\documentclass[conference]{IEEEtran}
\ifCLASSINFOpdf
\else
\fi
\usepackage{amsmath,amsfonts,amssymb,amsthm}

\usepackage{algorithm,algorithmic}
\usepackage{graphicx}

\newtheorem{theorem}{Theorem}
\newtheorem{lemma}{Lemma}
\newtheorem{corollary}{Corollary}

\DeclareMathOperator{\var}{Var}

\DeclareMathOperator*{\argmax}{arg\,max}
\DeclareMathOperator{\tr}{tr}
\DeclareMathOperator{\vol}{vol}
\DeclareMathOperator{\sign}{sign}
\DeclareMathOperator{\E}{\mathbb{E}}
\DeclareMathOperator{\T}{\top}
\allowdisplaybreaks[3]

\usepackage{amsmath,amssymb,amsfonts}

\begin{document}
%
\title{An Impossibility Result for High Dimensional Supervised Learning}

\author{
  \IEEEauthorblockN{M.~H.~Rohban, P.~Ishwar,
    B.~Orte$\mathrm{n}^\dagger$, W.~C.~Karl, and V.~Saligrama}
  \IEEEauthorblockA{
    ECE Department Boston University, \hfill{$\dagger$Turn, Inc.~CA USA} \\
    Email: \{mhrohban,pi,srv,wckarl\}@bu.edu, \hfill{burkay.orten@turn.com}}
}


%


\maketitle

\begin{abstract}
We study high-dimensional asymptotic performance limits of binary
supervised classification problems where the class conditional
densities are Gaussian with unknown means and covariances and the
number of signal dimensions scales faster than the number of labeled
training samples. We show that the Bayes error, namely the minimum
attainable error probability with complete distributional knowledge
and equally likely classes, can be arbitrarily close to zero and yet
the limiting minimax error probability of every supervised learning
algorithm is no better than a random coin toss. In contrast to related
studies where the classification difficulty (Bayes error) is made to
vanish, we hold it constant when taking high-dimensional
limits. 
In contrast to VC-dimension based minimax lower bounds that consider
the worst case error probability over {\em all} distributions that
have a fixed Bayes error, our worst case is over the family of
Gaussian distributions with constant Bayes error.
We also show that a nontrivial asymptotic minimax error probability
can only be attained for parametric subsets of zero measure (in a
suitable measure space). These results expose the fundamental
importance of prior knowledge and suggest that unless we impose strong
structural constraints, such as sparsity, on the parametric space,
supervised learning may be ineffective in high dimensional small
sample settings.
\end{abstract}


%
\IEEEpeerreviewmaketitle

\section{Introduction}

In a number of applications ranging from medical imaging to economics,
one encounters inference problems that suffer from the ``curse of
dimensionality'', namely the situation where the observed signals are
high-dimensional and we lack sufficient labeled training samples from
which to accurately learn models and make reliable decisions. This may
be true even when the underlying decision problem becomes ``easy''
with perfect knowledge of models or latent variables. For example, in
detecting the presence or absence of stroke, high-dimensional
tomographic X-ray projections are measured, though a stroke may affect
only tissue properties in a {\it localized} spatial region and may be
easily detectable if one knew where to look. Labeled training samples
are typically limited in this context due to the high cost of engaging
domain experts. Research in recent years has therefore focused on
leveraging prior knowledge in the form of sparsity or other latent
low-dimensional structure to improve decision making.

Our aim in this work is to expose certain fundamental limitations of
supervised learning in high dimensional small sample settings and
highlight the fundamental necessity of strong structural constraints
(prior knowledge), such as sparsity, for attaining nontrivial
asymptotic error rates. Towards this end we study the high-dimensional
asymptotic performance limit of binary supervised classification where
the class conditional densities are Gaussian with unknown means and
covariances. If $d$ and $n$ respectively denote the number of signal
dimensions and the number of labeled training samples, our focus is on
the ``big $d$ small $n$'' asymptotic regime where $n/d \rightarrow 0$
as $d \rightarrow \infty$. In previous related work, either $n/d
\rightarrow c > 0$ as $d \rightarrow \infty$ or the classification
difficulty (Bayes error) converges to zero as $d \rightarrow \infty$
\cite{Don04, Singh10}, or the focus was on special families of
learning rules such as Naive-Bayes, banded covariance structure
\cite{Bickel04}, and plug-in rules \cite{Shao12, Orten11}.  In
contrast, we hold the Bayes error constant when taking
high-dimensional limits.

We establish two key results in this paper: 1) An impossibility
result: When the number of signal dimensions scales faster than the
number of labeled samples at constant classification difficulty, the
asymptotic minimax classification error probability of {\it any
  supervised classification algorithm} cannot converge to anything
less than half. 2) Necessity of ``structure'' in parameter set:
Nontrivial asymptotic minimax error probability is attainable only for
parametric subsets of zero Haar measure.

\section{Related results from VC Theory}
There are several well known probabilistic lower bounds for the error
probability in the {\em distribution-free} setting of learning
\cite{Dev95}. These bounds are typically based on the VC dimension $V$
of a set of classifiers containing the optimal Bayes rule. For the
case when the Bayes error $P^{*}_e$ is zero, it is known that if $n <
(V-1)/(32 \epsilon)$, then for any supervised classifier $g_n$ and all
$\epsilon \leq 1/8$ and $\delta \leq 0.01$, $\sup_{p_{X, Y}: P^{*}_e =
  0} \Pr( L_{g_n} \geq \epsilon ) \geq \delta$ \cite{Ehr89}, where
$p_{X, Y}$ is the joint distribution of data points and their labels,
$L_{g_n} = \Pr(g_n(X) \neq Y | X_1, Y_1, \ldots, X_n, Y_n)$, and $\{
(X_1, Y_1), \ldots, (X_n, Y_n) \}$ is the training set. This bound
implies that if $n$ is small compared to $V$, then there exist
distributions for which, with probability of at least $0.01$, the
conditional Bayes error of every classifier is larger than $1/8$. For
the case when $P^{*}_e = c \neq 0$, it is known that if $n < V/(320
\epsilon)$ then for all supervised learning rules $g_n$, and all
$\epsilon , \delta \in (0,1/64)$, $\sup_{p_{X, Y}: P^{*}_e = c}
\Pr(L_{g_n} - c \geq \epsilon ) \geq \delta$. For linear classifiers
in $\mathbb{R}^d$, $V = (d + 1)$ and the mentioned bounds prove the
impossibility of learnability in the \emph{distribution-free} high
dimensional setting. However, these bounds are known to be pessimistic
as they are proved by constructing pathological adversarial
distributions $p_{X|Y}$ whose support is concentrated on $V$ points in
order to make the error of any learning rule in the hypothesis space
larger than $\epsilon$ with probability of at least $\delta$. It has
been suggested that these bounds do not hold for some practical choices of
$p_{X, Y}$ \cite{Oblo92}.

There also exist impossibility results for the {\it fixed
  distribution} setting \cite{Bene91} where a fixed and known
distribution $p_{X}$ is assumed and classifiers are assumed to belong
to a specific set $\mathcal{C}$. For $\Pr( L_{g_n} \leq \epsilon )
\geq 1-\delta$ it is {\em necessary} that $n \geq \log ((1 - \delta)
n_c)$ where $n_c$ is the $\epsilon$ covering number of $\mathcal{C}$
with respect to $p_X$ \cite{Bene91}.

The learning scenario discussed in this paper differs from both these
settings. We consider $p_{X|Y, \theta}$ to be a \emph{Gaussian
  distribution with an unknown set of parameters $\theta$}. Hence,
unlike the second setting, $p_X$ is not a fixed distribution, but
belongs to a \emph{family} of Gaussian distributions parameterized by
different choices of $\theta$. However, this family is much more
restricted than the set of {\em all} distributions, which is assumed
in the distribution-free setting. In addition, we establish a stronger
notion of impossibility that corresponds to taking $\epsilon =
\frac{1}{2}$ and $\delta = 1$ asymptotically, i.e., the worst-case
error probability within the family is not less than half
asymptotically. It should also be noted that our notion of
impossibility is not available in the distribution-free setting, and
additional effort is required to establish it in the
fixed-distribution setting.

\section{Problem Formulation}

\noindent {\bf Binary classification with Gaussian class
  conditional densities:} 
Let $X \in \mathbb{R}^d$ denote the observed signal (data), $Y \in
\{-1,+1\}$ the latent binary class label with $p_Y(+1) = p_Y(-1) =
1/2$, and $p_{X|Y,\theta}(x|y,\theta) = \mathcal{N}(\mu_y,\Sigma)(x)$,
where $\mu_{+1}, \mu_{-1} \in \mathbb{R}^d$, $\mu_{+1} \neq \mu_{-1}$
are mean vectors for classes $+1$ and $-1$ respectively, $\Sigma$ is a
covariance matrix common to both classes, and $\theta :=
(\mu_{+1},\mu_{-1},\Sigma)$ denotes the tuple of parameters. Thus
$(X,Y)|\theta \sim p_{X|Y,\theta}(x|y,\theta)p_Y(y)$. The minimum
probability of error classification rule, henceforth referred to as
the {\it optimum} rule, is the maximum aposteriori probability
(\text{MAP}) rule and is given by
\begin{equation} \label{MAPR}
\begin{split}
\widehat{y}^{*}(x) &=  \argmax_{y \in \{-1,+1\}}
p_{Y|X,\theta}(y|x,\theta) \\
&= \sign \left(\Delta^T\Sigma^{+}(x - \mu)\right)
\end{split}
\end{equation}
where $\mu := \frac{1}{2}(\mu_{+1} + \mu_{-1})$, $\Delta := \mu_{+1} -
\mu_{-1}$, and $\Sigma^{+}$ is the pseudoinverse of $\Sigma$. The
error probability of the optimum rule (Bayes error) is given by
\begin{equation} \label{BRisk}
P_e^{*} =
Q\left(\frac{1}{2}\left\|\left(\Sigma^{+}\right)^{\frac{1}{2}}\left(\mu_{+1}-\mu_{-1}\right)\right\|\right)
=: Q(\alpha/2)
\end{equation}
where $Q(t) := \int_t^\infty
\frac{1}{\sqrt{2\pi}}\exp\{-\frac{1}{2}t^2\} dt$, $\alpha =
\|(\Sigma^{+})^{\frac{1}{2}} (\mu_{+1} - \mu_{-1})\|$, and the minimum
attainable error probability with complete distributional knowledge
and equally likely classes, i.e., the ``classification difficulty'',
is given by $Q(\alpha/2)$. We will assume that $\alpha > 0$ so that
the Bayes error is nontrivial, i.e, strictly less than $1/2$. \\

\noindent {\bf Supervised classification rules:} Let $\mathcal{T}_n :=
\{(X_i,Y_i), i = 1,\ldots,n\}$ be a set of $n$ labeled training
samples that are independent and identically distributed (iid)
according to $p_{X|Y,\theta} \cdot p_Y$ where $\theta$ belongs to a
{\it known} set of feasible parameters $\Theta$.  Let $X_0$ be a test
sample (independent of $\mathcal{T}_n$) whose {\it true} class label
$Y_0$ is unobservable and needs to be estimated. A supervised
classification rule is a measurable mapping
\[
\widehat{y}_{n,d}: \mathbb{R}^d \times \mathcal{T}_n \rightarrow
\{-1,+1\}
\]
from the test data space to the set of class labels constructed using
an algorithm that has access to the set of training samples
$\mathcal{T}_n$ and knowledge of the form of
$p_{X,Y|\theta}(x,y|\theta)$ and $\Theta$ but no direct knowledge of
$\theta$ itself. \\

\noindent {\bf Constant difficulty parameter sets:} We are interested
in taking constant-difficulty high-dimensional limits. To this end we
define $\Theta_0(\alpha)$ to be the set of all $\theta$ values for
which the Bayes error is equal to $Q(\alpha/2)$ (see \eqref{BRisk}):
\begin{equation*}
\Theta_0(\alpha) :=
\left\{ \left(\mu_{+1},\mu_{-1},\Sigma\right): \left\|\ \left(
\Sigma^{+} \right)^{\frac{1}{2}} \left(\mu_{+1} -
\mu_{-1}\right)\right\| = \alpha
\right\}.
\end{equation*}


A canonical subset of $\Theta_0$ of particular interest is one in
which the covariance is spherical and the means are constrained to be
on opposites ends of a $d$-dimensional sphere:
\begin{equation*}
\Theta_{\text{Sphere}}(\alpha) :=
\left\{\left({\mathbf h}, -{\mathbf h},\beta^2 {\mathbf I} \right):
\lVert {\mathbf h} \rVert = 1, \beta = 2/\alpha
\right\}.
\end{equation*}
Clearly, $\Theta_{\text{Sphere}}(\alpha) \subseteq
\Theta_0(\alpha)$. This special parametric set corresponds to the
scenario in which $X$ can be represented as $X = Y {\mathbf h} + Z$
where $Z$ is white Gaussian noise that is independent of $Y$. \\

%

\noindent{\bf Error probabilities:} 
%
%
%
Let
\[
P_{e|\theta}(\widehat{y}_{n,d}) :=
%
%
\Pr\left(\widehat{y}_{n,d}(X_0) \neq Y_0 \right)
\]
denote the expected error probability of the classifier
$\widehat{y}_{n,d}$ averaged across training samples
$\mathcal{T}_n$ for some parameter $\theta$
and let
\[
P_{e}(n,d,\Theta, \widehat{y}_{n,d}) := \sup_{\theta \in
  \Theta} P_{e|\theta}(\widehat{y}_{n,d})
\]
be the maximum expected error probability of the classifier
$\widehat{y}_{n,d}$ over the parameter set $\Theta$ which
depends on the number of labeled training samples $n$ and the number
of signal dimensions $d$. \\

\noindent{\bf Goal:} We aim to gain an understanding of how
$P_{e}(n,d,\Theta, \widehat{y}_{n,d})$ behaves in the
constant-difficulty high dimensional setting where $d,n \rightarrow
\infty$, $n/d, n/\text{rank}(\Sigma^{+}) \rightarrow 0$, and $\Theta
\subseteq \Theta_0(\alpha)$ is a sequence of constant-difficulty
parameter sets.
%

\section{Main Results}

\begin{theorem} \label{mainImpossResult}
For any sequence of classifiers $\widehat{y}_{n,d}$, we have
\begin{equation*}
\liminf\limits_{(d,n/d)\rightarrow(\infty,0)}
P_{e}(n,d,\Theta_{\text{Sphere}}, \widehat{y}_{n,d}) \geq
\frac{1}{2}
\end{equation*}
\end{theorem}

\begin{corollary}
For any sequence of parameter sets $\Theta$ with
$\Theta_{\text{Sphere}} \subseteq \Theta$, and any sequence of
classifiers $\widehat{y}_{n,d}$, we have
\begin{equation*}
\liminf\limits_{(d, n/d) \rightarrow (\infty, 0)} P_e(n, d, \Theta,
\widehat{y}_{n,d}) \geq \frac{1}{2}
\end{equation*}
\end{corollary}

\begin{corollary}
Let $\Theta_{\text{subset}} := \{ ({\mathbf h}, -{\mathbf h}, \beta^2
{\mathbf I}) \in \Theta_{\text{Sphere}}, \mathbf{h} \in \mathcal{H}
\subseteq \mathcal{S}^{d-1} \}$ where $\mathcal{S}^{d-1}$ is the unit
$(d-1)$-sphere in $\mathbb{R}^d$. Let $\vol({\mathcal{H}}) \triangleq
\Pr_{H \sim U(\mathcal{S}^{d-1})}(H \in \mathcal{H})$, where
$U(\mathcal{S}^{d-1})$ denotes the uniform distribution over
$\mathcal{S}^{d-1}$. If for a sequence of classifiers
$\widehat{y}_{n,d}$,
\[
\limsup\limits_{(d, n/d) \rightarrow (\infty, 0)} P_e(n, d,
\Theta_{\text{Sphere}}, \widehat{y}_{n,d}) = \frac{1}{2}
\]
and 
\[
\lim_{d\rightarrow\infty}\vol({\mathcal{H}}) > 0
\] 
then
\[
\limsup\limits_{(d, n/d) \rightarrow (\infty, 0)} P_e(n, d,
\Theta_{\text{subset}}, \widehat{y}_{n,d}) \geq \frac{1}{2}.
\]
%
%
\end{corollary}

The proofs of the theorem and the two corollaries are presented in
Section~\ref{ProofSec}.

\section{Discussion}
\begin{figure}[!htb]
\begin{center}
\includegraphics[scale=0.40]{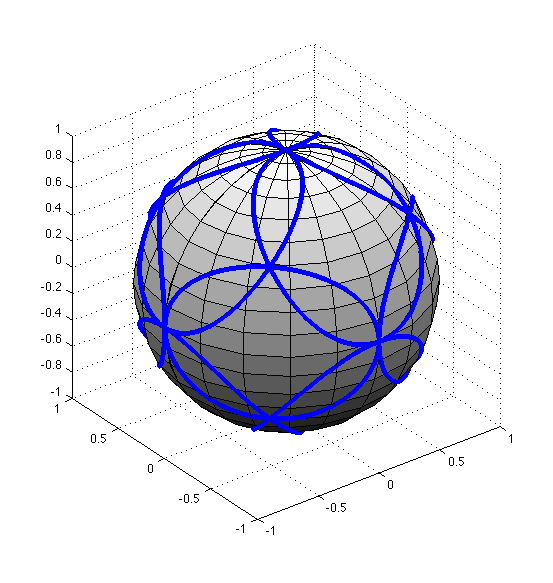}\\
\vglue -3ex
\includegraphics[scale=0.40]{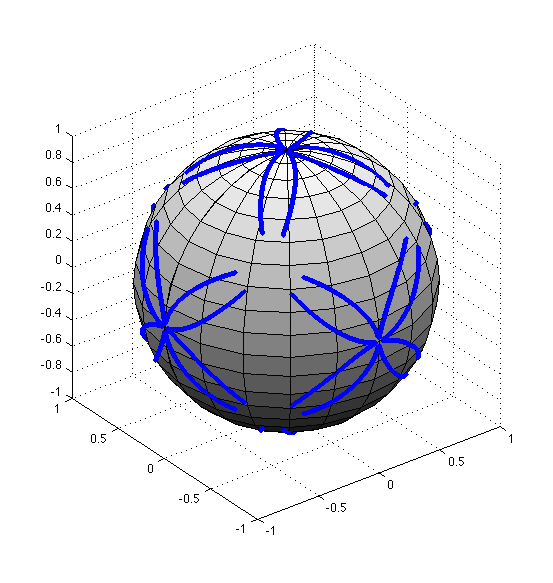}
\end{center}
\vglue -6ex
\caption{Exponential sparsity class $\mathcal{H}_{exp}$ (solid curves,
  top figure) and polynomial sparsity class $\mathcal{H}_{poly}$
  (solid curves, bottom figure) for $d = 3$.}
\label{f11}
\end{figure}
%

Theorem~\ref{mainImpossResult} informs us that even in the ``easier''
scenario where the covariance is spherical and perfectly known, the
worst case classification performance of {\it any} sequence of
supervised classifiers is asymptotically no better than the classifier
which flips an unbiased coin to make decisions. It is interesting to
note that this conclusion holds for any arbitrarily small positive
Bays error $Q(\alpha/2)$.

An immediate corollary of this theorem is the impossibility of
attaining less than half asymptotic error probability for larger
parametric sets, specifically parameter sets that contain
$\Theta_{\text{Sphere}}$ such as $\Theta_0$. These parameter sets
contain more unknowns (degrees of freedom), e.g., the covariance, but
have the same classification difficulty $Q(\alpha/2)$.

These results are consistent with previous results concerning the
asymptotic behavior 
of the so-called plug-in family of classifiers where Maximum
Likelihood (ML) estimates of parameters are plugged into the MAP rule
given in \eqref{MAPR}):
\begin{itemize}
\item Plug-in classification for $\Theta_{0}$: In \cite{Orten11} it was
  shown that for plug-in classifiers that use ML estimates of the means
  and covariance, the performance is asymptotically no better than
  half.
\item Sensing-aware classification: Here it is assumed that data is
  generated according to $X = {\mathbf h} W + Z$, where ${\mathbf h}$
  is the ``sensing subspace'', $W$ is a class-dependent scalar latent
  variable and $Z$ is white Gaussian noise independent of the class
  labels and latent variables. Assuming Gaussian class conditional
  densities for $W$, the class conditional density of $X$ would be
  Gaussian with covariance $\gamma^2 {\mathbf h} {\mathbf h}^{\T} +
  \beta^2 {\mathbf I}$. Also the mean of $X$ would be in the same
  direction ${\mathbf h}$ for the two classes. In \cite{Orten11} it
  was shown that the asymptotic probability of error of the ML
  projection classifier which is based on projecting data along the ML
  estimate of ${\mathbf h}$ is also no better than half. This can be
  seen as a special case of the Corollary 1, by considering
  ${\Theta_{\text{Sensing Aware}}} \supseteq \Theta_{\text{Sphere}}$
  where
\begin{multline*}
{\Theta_{\text{Sensing Aware}}}(\alpha) := \big\{\left(m_{1} {\mathbf
  h} ,m_{2} {\mathbf h}, \gamma^2 {\mathbf h} {\mathbf h}^{\T} +
\beta^2 {\mathbf I} \right): \\ \|{\mathbf h}\| = 1,
 \gamma \geq 0, \beta > 0,
|m_{1} - m_{2}| = \alpha\sqrt{\gamma^2+\beta^2}
\big\}.
\end{multline*}
\end{itemize}

Another important consequence of the Theorem \ref{mainImpossResult} is
Corollary~2. This result can be interpreted as implying that
nontrivial asymptotic minimax error probability is attainable only for
parametric subsets of zero Haar measure.
%
%
This highlights the {\it necessity} of some weak form of sparsity
in the set of feasible ${\mathbf h}$ values in order to attain
non-trivial asymptotic error probability. This is consistent with
previous results that have shown that the supervised classification
error can in fact converge to the Bayes error when ${\mathbf h}$
belongs to specific sparsity classes \cite{Shao12, Orten11}.

Specifically, in \cite{Orten11} it was shown that if the magnitudes of
the components of ${\mathbf h}$ decay exponentially (or even
polynomially) when reordered according to decreasing values, then
${\mathbf h}$ can be estimated consistently (in the mean square sense)
by soft-thresholding the ML estimate of ${\mathbf h}$ even in the case that $d$ 
grows sub-exponentially in terms of $n/\log(n)$. Then it can be shown that a classifier
based on projecting data onto the estimation of ${\mathbf h}$ attains the Bayes error
asymptotically.
%

To represent these two sparsity classes, $\mathcal{H}_{exp}$ and
$\mathcal{H}_{poly}$ were defined in \cite{Orten11} as bellow:
\begin{gather*}
\mathcal{H}_{exp} = \left\{ {\mathbf h} : \left| h_{(k)} \right| =
M_1(d) \alpha^k, 0 < \alpha < 1 \right\} \\
\mathcal{H}_{poly} = \left\{ {\mathbf h} : \left| h_{(k)} \right| =
M_2(d) k^{-\beta}, \beta > 0.5 \right\}
\end{gather*}
where $(h_{(1)}, \ldots, h_{(d)})$ are the components of ${\mathbf h}$
in decreasing order of magnitude.
%
%
Since there is only one degree of freedom in each set, the Haar
measure of these two sets vanish when $d \geq 3$. \hbox{Figure
  \ref{f11}} illustrates these two sets (dark solid curves) on the
unit sphere in 3-dimensional space. These sets satisfy the necessary
conditions suggested by the Corollary 2 and in fact achieve the Bayes
error asymptotically which is better than half.
To summarize, it is essential to have strong structural assumptions on
the feasible set of parameters in order to obtain non-trivial
asymptotic classification performance in high-dimensional small sample
settings.

\section{Proofs} \label{ProofSec}

\subsection{Proof of Theorem 1}
\begin{proof}
%
The key proof-idea is to randomize the selection of $\theta \in
\Theta_{\text{Sphere}}$. The worst error probability over $\theta \in
\Theta_{\text{Sphere}}$ is not smaller than the average (expected)
error probability when $\theta$ is random. For any fixed value of
$\theta$, the training and test samples are totally independent but if
$\theta$ is random, they can become dependent {\it through}
$\theta$. Then the expected error probability (with respect to both
training data and $\theta$) of any supervised classification rule
$\widehat{y}_{n,d}$ cannot be smaller than that of the
$\theta$-distribution-induced MAP rule based on $\mathcal{T}_n$. If
the randomizing distribution for $\theta$ is carefully selected then
the lower bound can be made to converge to $1/2$ as $n$ and $d$
scale. Now we work out the details of this proof-idea.

Let $\mathcal{S}^{d-1}$ denote the unit $(d-1)$-sphere. For every
\hbox{${\mathbf h} \in \mathcal{S}^{d-1}$}, $\| {\mathbf h} \|=1$ and
$({\mathbf h},-{\mathbf h}, \beta^2 {\mathbf I}) \in
\Theta_{\text{Sphere}}$. Let \hbox{$\theta \sim (H,-H, \beta^2
  {\mathbf I})$} where $H \sim \mathrm{Uniform}(\mathcal{S}^{d-1})$
which is independent of data labels $Y_i$. Then, for every realization
$H= {\mathbf h}$, $\theta \in \Theta_{\text{Sphere}}$ and conditioned
on $H = {\mathbf h}$ (equivalently conditioned on a realization of
$\theta$), the training and test samples have the joint distribution
that was described earlier. By defining
\begin{equation}
\widehat{y}_{\text{MAP}}(X_0) \triangleq \argmax_{y_0\in\{-1,+1\}}
p_{Y_0|X_0,\mathcal{T}_n}(y_0|x_0,\mathcal{T}_n,
\Theta_{\text{Sphere}})
\end{equation}
and using the notation $P_e \triangleq P_{e}(n, d,
\Theta_{\text{Sphere}}, \widehat{y}_{n,d})$, we have :
\begin{equation}  \label{PeSimplify}
\begin{split}
P_{e} &\geq \E_{\theta \in
  \Theta_{\text{Sphere}}}\left[P_{e|\theta}(\widehat{y}_{n,d})
  \right] \\
%
%
&= \E_{\theta \in
  \Theta_{\text{Sphere}},\mathcal{T}_n}\left[\Pr\left(\widehat{y}_{n,d}(X_0)
  \neq Y_0| \mathcal{T}_n, \theta \right)\right] \\
& \stackrel{(i)}{\geq} \E_{\theta \in \Theta_{\text{Sphere}},
  \mathcal{T}_n}\left[\Pr\left(\widehat{y}_{\text{MAP}}(X_0) \neq Y_0|
  \mathcal{T}_n, \theta \right)\right]
\end{split}
\end{equation}
%
where $(i)$ holds because the MAP rule minimizes the error
probability.
%
%
%
%
Using the notation $\widehat{y}_{\text{MAP}} \triangleq
\widehat{y}_{\text{MAP}}(X_0)$, we have:
%
%
\begin{align*}
\widehat{y}_{\text{MAP}} &= \argmax_{y_0\in\{-1,+1\}}
p_{Y_0|X_0,\mathcal{T}_n}(y_0|x_0,\mathcal{T}_n,
\Theta_{\text{Sphere}}) \\
&= \argmax\limits_{y_0\in\{-1,+1\}}
p_{X_0,Y_0,\mathcal{T}_n}(x_0,y_0,\mathcal{T}_n |
\Theta_{\text{Sphere}}) \\
&= \argmax_{y_0\in\{-1,+1\}} \E_{\theta \in
  \Theta_{\text{Sphere}}}\left[p_{X_0,Y_0,\mathcal{T}_n|\theta}(x_0,y_0,\mathcal{T}_n|\theta)\right]\\
&\stackrel{(i)}{=} \argmax_{y_0\in\{-1,+1\}} \E_{\theta \in
  \Theta_{\text{Sphere}}}\left[\prod_{i=0}^{n}p_{X_i,Y_i|\theta}(x_i,y_i|\theta)\right]\\
&\stackrel{(ii)}{=} \argmax_{y_0\in\{-1,+1\}} \E_{\theta \in
  \Theta_{\text{Sphere}}}\left[\prod_{i=0}^{n}p_{X_i|Y_i,\theta}(x_i|y_i,\theta)\right]\\
&\stackrel{(iii)}{=} \argmax_{y_0\in\{-1,+1\}}
\E_H\left[\prod_{i=0}^{n}\exp\left\{ -\frac{1}{2 \beta^2}\left\|x_i-
  y_i H\right\|^2\right\}\right] \\
&= \argmax_{y_0\in\{-1,+1\}} \E_H\Big[\exp\Big\{ -\frac{1}{2
    \beta^2}\sum_{i=0}^{n} \left\|x_i\right\|^2 +
  \\ &~~~~~~~~~~~~~~~~~~~~~~~~ \left\| y_iH\right\|^2 -2 y_ix_i^TH
  \Big\}\Big]\\
&\stackrel{(iv)}{=} \argmax_{y_0\in\{-1,+1\}} \E_H\left[\exp\left\{
  \frac{1}{\beta^2} H^T \left(\sum_{i=0}^{n}
  y_ix_i\right)\right\}\right]\\
&\stackrel{(v)}{=} \argmax_{y_0\in\{-1,+1\}} \frac{1}{\beta^2}
\left\|\sum_{i=0}^{n}y_ix_i \right\| \\
&= \sign \left(x_0^{\T} \left(\sum_{i=1}^n y_ix_i\right)\right).
\end{align*}
In the above derivation, $(i)$ is because the training and test
samples are conditionally independent given $\theta$, $(ii)$ is
because $p_{Y_i}(+1) = p_{Y_i}(-1) = 0.5$ for all $i$, $(iii)$ is
because $\theta = (H,-H, \beta^2 {\mathbf I}) =
(\mu_{+1},\mu_{-1},\Sigma)$, $(iv)$ is because $\|H\| = 1$ and $y_i =
\pm 1$ for all $i$, and $(v)$ is due to the following result. 

\begin{lemma}
Let $H$ be uniformly distributed on $\mathcal{S}^{d-1}$ and $f(x) :=
\E_H[\exp\{H^{\T} x\}]$. Then $f(x)$ is a radial function that is convex
and nondecreasing in $\|x\|$.
\end{lemma}
\begin{proof}
Since $H$ is uniformly distributed on $\mathcal{S}^{d-1}$, $f(x)$ is a
radial function. Consider $x = t u$ where $t \in \mathbb{R}$ and
$\|u\| = 1$. If $g(t) := f(tu)$ then $g(0) = 1$ and $g$ is symmetric
since the distribution of $H$ is spherically symmetric. We also have
$g^{\prime}(t) = \E_H[(H^{\T}u)\exp\{tH^{\T}u\}]$ so that $g^{\prime}(0) =
0$, again because the distribution of $H$ is spherically
symmetric. Finally, $g^{\prime\prime}(t) = \E_H[(H^{\T}u)^2\exp\{tH^{\T}u\}]
\geq 0$ which shows that $g(t)$ is convex for $t \geq 0$. Since $g(t)$
is convex for $t \geq 0$ and $g^{\prime}(0) =0$, it is nondecreasing
for $t \geq 0$.
\end{proof}
\noindent Continuing the proof, we have : 
\begin{equation} \label{LastEq}
\begin{split}
P_{e} &\stackrel{(i)}{\geq}
\E_{\theta,\mathcal{T}_n}\left[P\left(\widehat{y}_{\text{MAP}}(X_0,\mathcal{T}_n)
  \neq Y_0| \mathcal{T}_n, \theta\right)\right] \\
&= \E_{H,\mathcal{T}_n}\bigg[ \Pr\bigg(\sign \Big(X_0^{\T}
  \sum_{i=1}^n Y_iX_i \Big)\neq Y_0\bigg| \mathcal{T}_n, H\bigg)\bigg]
\\
&= \E_{H, \mathcal{T}_n}\left[Q\left(\frac{-H^{\T}\left(\sum_{i=1}^n
    Y_iX_i\right)}{\beta \|\left(\sum_{i=1}^n
    Y_iX_i\right)\|}\right)\right] \\
&= \E_{H, \mathcal{T}_n}\left[Q\left(\frac{-(1 + H^{\T}V)}{\beta
    \sqrt{1 + 2H^{\T}V + \|V\|^2}}\right)\right]
\end{split}
\end{equation}
where $V \sim \mathcal{N}({\mathbf 0},\frac{\beta^2}{n} {\mathbf
  I}_d)$ and is independent of $H$.  This follows as we can write
$X_i$ as $X_i = Y_i H + Z_i$, where $Z_i$ are white Gaussian noise
with variance $\beta^2$, which are independent of $H$ and labels
$Y_i$. Hence $\frac{1}{n} \sum_{i = 1}^{n} X_i Y_i = H +
\frac{1}{n}\sum_{i = 1}^{n} Y_i Z_i$. Finally, by taking \hbox{$V =
  \frac{1}{n}\sum_{i = 1}^{n} Y_i Z_i$}, it follows that $V \sim
\mathcal{N}({\mathbf 0},\frac{\beta^2}{n} {\mathbf I}_d)$.  Note that
$(i)$ is proved in equation \eqref{PeSimplify}. \\

\begin{lemma}
$W = \frac{1 + H^{\T} V}{\sqrt{1 + 2H^{\T}V + \|V\|^2}}
  \xrightarrow{p} 0$, as $d \rightarrow \infty$.
\end{lemma}
\begin{proof}
Since $H$ and $V$ are independent, $\E(H^{\T} V) \stackrel{}{=}
\E(H^{\T}) \E(V) = 0$ and
\begin{equation*}
\begin{split}
\var(H^{\T}V) &= \E(H^{\T} V V^{\T} H) = \E_{H} \E_{V | H} (H^{\T} V
V^{\T} H) \\
&\stackrel{}{=} \E_H \left(\frac{\beta^2}{n} H^{\T} H \right) =
\frac{\beta^2}{n} \longrightarrow 0
\end{split}
\end{equation*}
as $n,d\rightarrow \infty$.
As a result $1+H^{\T} V \xrightarrow{p} 1$. Next, we will show that
$\var(1 + 2H^{\T}V + \|V\|^2) =
\mathcal{O}\left(\frac{d}{n^2}\right)$. First, observe that $\E(1 +
2H^{\T}V + \|V\|^2) = 1 + \beta^2 \frac{d}{n}$. Thus
%
\begin{equation*}
\begin{split}
\var(1 + 2H^{\T}V + \|V\|^2) = & \E \left((1 + 2H^{\T}V + \|V\|^2)^2
\right) \\
& ~~~ - \left(1 + \beta^2 \frac{d}{n}\right)^2
\end{split}
\end{equation*}
Furthermore, we have
\begin{multline*}
%
%
\E \left((1 + 2H^{\T}V + \|V\|^2)^2 \right) = \underbrace{4\E(H^{\T}
  V V^{\T} H)}_{4 \beta^2/n} + \E(\|V\|^4) \\
+ \underbrace{4\E(H^{\T} V)}_0 + \underbrace{2\E(\| V \|^2)}_{2
  \beta^2 \frac{d}{n}} 
+ \underbrace{4\E(H^{\T} V \| V \|^2)}_{0} + 1
%
%
\end{multline*}
It remains to calculate $\E(\| V \|^4)$. Note that $\E(\| V \|^4) =
\var(V^{\T} V) + \beta^4 \frac{d^2}{n^2} $. But we know that for a
Gaussian random variable $\epsilon \sim \mathcal{N} (\mu, \Sigma)$,
and an arbitrary matrix $\Lambda$, we have $\var(\epsilon^{\T} \Lambda
\epsilon) = 2 \tr (\Lambda \Sigma \Lambda \Sigma) + 4 \mu^{\T} \Lambda
\Sigma \Lambda \mu$. Therefore, \hbox{$\E(\| V \|^4) = 2 \beta^4
  \frac{d}{n^2} + \beta^4 \frac{d^2}{n^2}$.}  Finally, we have
\begin{equation*}
\var(1 + 2H^{\T}V + \|V\|^2) = 4 \frac{\beta^2}{n} + 2 \beta^4
\frac{d}{n^2} = \mathcal{O}\left( \frac{d}{n^2} \right)
\end{equation*}
As a result, $\frac{n}{d} (1 + 2H^{\T}V + \|V\|^2) \xrightarrow{p}
\beta^2$, because it converges in $L^2$.  We have
\begin{equation*}
W = \sqrt{\frac{n}{d}} \frac{1 + H^{\T} V} {\sqrt{\frac{n}{d} \left( 1
    + 2 H^{\T} V + \| V \|^2 \right)}}
\end{equation*}
Note that the numerator and denominator go to 1 and $\beta$ in
probability, respectively. Therefore, using Slutsky's Theorem, the
whole fraction goes to $1/ \beta$ in probability. But
$\sqrt{\frac{n}{d}}$ goes to zero, therefore, $W \xrightarrow{p} 0$.
\end{proof}
Using Slutsky's theorem, \hbox{$Q\left(-W/\beta \right)
  \xrightarrow{p} \frac{1}{2}$.} Since \hbox{$0 \leq Q(.) \leq 1$},
using Dominated Convergence Theorem, $\lim\limits_{(d, n/d)
  \rightarrow (\infty, 0)} \E\left[ Q\left(-W/\beta^2 \right) \right]
= \frac{1}{2}$. Taking limit inferior of both sides of \eqref{LastEq},
we finally conclude that for any $\widehat{y}_{n,d}$
\begin{equation*}
\liminf\limits_{(d, n/d) \rightarrow (\infty, 0)} P_e(n, d,
\Theta_{\text{Sphere}}, \widehat{y}_{n,d}) \geq \frac{1}{2}
\end{equation*}
%
%
\end{proof}
\vspace{-1mm}
\subsection{Proof of Corollary 1}
%
%
For any classifier $\widehat{y}_{n,d}$, we have :
\begin{equation*}
\begin{split}
P_e(n, d, \Theta, \widehat{y}_{n,d}) & = \sup\limits_{\theta
  \in \Theta} P_{e | \theta}(\widehat{y}_{n,d}) \\
 & ~{\geq} \sup\limits_{\theta \in \Theta_{\text{Sphere}}} P_{e |
  \theta}(\widehat{y}_{n,d}) = P_e(n, d,
\Theta_{\text{Sphere}}, \widehat{y}_{n,d})
\end{split}
\end{equation*}
because $\Theta_{\text{Sphere}} \subseteq \Theta$. By taking limit
inferior of two sides, the Corollary is proved.
\subsection{Proof of Corollary 2}
%
%
Let \hbox{$\vol({\mathcal{H}}) \triangleq \Pr_{H \sim
    U(\mathcal{S}^{d-1})}(H \in \mathcal{H})$}, where
$U(\mathcal{S}^{d-1})$ denotes the uniform distribution over the unit
\hbox{$(d-1)$}-sphere in $d$ dimensional space. Suppose that
\hbox{$\limsup\limits_{(d, n/d) \rightarrow (\infty, 0)} P_e(n, d,
  \Theta_{\text{subset}}, \widehat{y}_{n,d}) <
  \frac{1}{2}$}. Hence \hbox{$\limsup\limits_{(d, n/d) \rightarrow
    (\infty, 0)} \E_{H \sim U(\mathcal{H})}(P_{e |
    \theta}(\widehat{y}_{n,d})) < \frac{1}{2}$}.  For the
specified sequence of classifiers $\widehat{y}_{n,d}$, which
satisfies the conditions of the corollary, we have
\begin{align*}
& \limsup\limits_{(d, n/d) \rightarrow (\infty, 0)} \E_{H \sim
    U(\mathcal{S}^{d-1})}(P_{e |
    \theta}(\widehat{y}_{n,d}))\\
& ~~~ = \limsup\limits_{(d, n/d) \rightarrow (\infty, 0)}
  \vol(\mathcal{H}) \E_{H \sim U(\mathcal{H})}(P_{e |
    \theta}(\widehat{y}_{n,d})) \\
& ~~~~~~~~ + \vol(\bar{\mathcal{H}}) \E_{H \sim
    U(\bar{\mathcal{H}})}(P_{e |
    \theta}(\widehat{y}_{n,d}))\\
& ~~~ \stackrel{}{\leq} \limsup\limits_{(d, n/d) \rightarrow
    (\infty, 0)} \hspace{-3 mm} \vol(\mathcal{H}) \hspace{-3 mm}
  \limsup\limits_{(d, n/d) \rightarrow (\infty, 0)} \hspace{-2 mm}
  \E_{H \sim U(\mathcal{H})}(P_{e |
    \theta}(\widehat{y}_{n,d})) \\
& ~~~~~~~~~~ + \hspace{-4 mm} \limsup\limits_{(d, n/d) \rightarrow
    (\infty, 0)} \hspace{-3 mm} \vol(\bar{\mathcal{H}}) \hspace{-3 mm}
  \limsup\limits_{(d, n/d) \rightarrow (\infty, 0)} \hspace{-2 mm}
  \E_{H \sim U(\bar{\mathcal{H}})}(P_{e |
    \theta}(\widehat{y}_{n,d})) \\
& ~~~ \stackrel{}{<} \frac{1}{2} \left( \lim\limits_{(d, n/d)
    \rightarrow (\infty, 0)} \vol(\mathcal{H}) + \lim\limits_{(d, n/d)
    \rightarrow (\infty, 0)} \vol(\bar{\mathcal{H}}) \right) \\
& ~~~~ {<} ~~ \frac{1}{2}
\end{align*}
which is a contradiction. 
%
\section{Concluding Remarks}
That prior knowledge such as sparsity improves inference in high
dimensional small sample settings is folklore. The results presented
here show that in fact such knowledge is absolutely indispensable in
that otherwise the asymptotic performance degenerates to a random coin
toss.  The results presented here focused on supervised binary
classification with Gaussian class conditional densities and equally
likely classes. One could expect similar conclusions to hold in more
complex inference problems. However the proof techniques used in this
work may not generalize to more complex and non-Gaussian settings.

\section*{Acknowledgment}
This article is based upon work supported by the U.S. AFOSR and
U.S. NSF under award numbers \#FA9550-10-1-0458 (subaward \#A1795) and
\#1218992 respectively.

\bibliographystyle{IEEEtran}
\bibliography{refs}

\end{document}